\def\year{2021}\relax
\documentclass[letterpaper]{article} % DO NOT CHANGE THIS
\usepackage{aaai21}  % DO NOT CHANGE THIS
\usepackage{times}  % DO NOT CHANGE THIS
\usepackage{helvet} % DO NOT CHANGE THIS
\usepackage{courier}  % DO NOT CHANGE THIS
\usepackage[hyphens]{url}  % DO NOT CHANGE THIS
\usepackage{graphicx} % DO NOT CHANGE THIS
\usepackage{natbib}  % DO NOT CHANGE THIS AND DO NOT ADD ANY OPTIONS TO IT
\usepackage{caption} % DO NOT CHANGE THIS AND DO NOT ADD ANY OPTIONS TO IT
\usepackage[utf8]{inputenc} % allow utf-8 input
\usepackage[T1]{fontenc}    % use 8-bit T1 fonts
\usepackage{hyperref}       % hyperlinks
\usepackage{url}            % simple URL typesetting
\usepackage{booktabs}       % professional-quality tables
\usepackage{amsfonts}       % blackboard math symbols
\usepackage{nicefrac}       % compact symbols for 1/2, etc.
\usepackage{microtype}      % microtypography
\usepackage{xcolor}         % colors

\usepackage{graphicx}
\usepackage{epstopdf}
\usepackage{mathtools}
\usepackage{amsthm, amssymb}
\usepackage[ruled, vlined, linesnumbered]{algorithm2e}
\usepackage{caption, subcaption}
\usepackage{enumitem}
\usepackage{color}
\usepackage{comment}

\newtheorem{theorem}{Theorem}
\newtheorem{lemma}[theorem]{Lemma}
\newtheorem{claim}[theorem]{Claim}
\newtheorem{corollary}[theorem]{Corollary}

\newtheorem{observation}[theorem]{Observation}

\newtheorem{remark}[theorem]{Remark}

\newtheorem{definition}{Definition}

\newtheorem{example}{Example}

\renewcommand{\v}{\mathbf{v}}
\newcommand{\w}{\mathbf{w}}
\renewcommand{\a}{\mathbf{a}}
\renewcommand{\b}{\mathbf{b}}
\newcommand{\g}{\mathbf{g}}
\newcommand{\Top}{\operatorname{\tt Top}}
\newcommand{\E}{\mathcal{E}}
\newcommand{\A}{\mathcal{A}}
\newcommand{\B}{\mathcal{B}}

\newcommand{\Q}{\mathcal{Q}}

\newcommand{\eps}{\epsilon}
\newcommand{\abs}[1]{\left| #1 \right|}

\DeclareMathOperator{\bE}{{\mathop{\mathbb{E}}}}
\DeclareMathOperator{\Geo}{{\mathop{\mathrm{Geo}}}}

\newcommand{\Prune}{\textsc{Prune}}
\newcommand{\Explore}{\textsc{Explore}}
\newcommand{\ExploreS}{\textsc{ExploreSet}}

\title{Instance-Sensitive Algorithms for Pure Exploration in Multinomial Logit Bandit\thanks{N. Karpov and Q. Zhang are supported in part by CCF-1844234 and CCF-2006591.}}

\author {
	Nikolai Karpov \quad \quad
	Qin Zhang  \\
}
\affiliations {
	 Indiana University Bloomington \\
	\{nkarpov, qzhangcs\}@indiana.edu
}

\begin{document}

\maketitle

\begin{abstract}
Motivated by real-world applications such as fast fashion retailing and online advertising, the Multinomial Logit Bandit (MNL-bandit) is a popular model in online learning and operations research, and has attracted much attention in the past decade.   In this paper, we give efficient algorithms for {\em pure exploration} in MNL-bandit.  Our algorithms achieve {\em instance-sensitive} pull complexities.  We also complement the upper bounds by an almost matching lower bound.
\end{abstract}

\section{Introduction}
\label{sec:intro}

We study a model in online learning called {\em multinomial logit bandit} ({\em MNL-bandit} for short), where we have $N$ substitutable items $\{1, 2, \ldots, N\}$, each of which is associated with a known reward $r_i \in (0, 1]$ and an {\em unknown} preference parameter $v_i \in (0, 1]$. We further introduce a null item $0$ with reward $r_0 = 0$, which stands for the case of ``no-purchase''. We set $v_0 = 1$, that is, we assume that the no-purchase decision is the most frequent case, which is a convention in the MNL-bandit literature and can be justified by many real-world applications to be mentioned shortly. 

Denote $[n] \triangleq \{1, 2, \ldots, n\}$.  Given a subset (called {\em an assortment}) $S \subseteq [N]$, the probability that one chooses $i \in S \cup \{0\}$ is given by 
$$
p_i(S) = \frac{v_i}{v_0 + \sum_{j \in S} v_j} = \frac{v_i}{1 + \sum_{j \in S} v_j}.
$$
Intuitively, the probability of choosing the item $i$ in $S$ is proportional to its preference $v_i$.  This choice model is called the {\em MNL choice model},  introduced independently by Luce~\cite{Luce59} and Plackett~\cite{Plackett75}. 
We are interested in finding an assortment $S \subseteq [N]$ such that the following expected reward is maximized.

\begin{definition}[expected reward]
\label{def:reward}
Given an assortment $S \subseteq [N]$ and a vector of item preferences $\v = (v_1, \ldots, v_N)$,  the expected reward of $S$ with respect to $\v$ is defined to be
\begin{equation*}
\label{eq:reward}
 R(S, \v) = \sum_{i \in S} r_i p_i(S) = \sum_{i \in S} \frac{r_i v_i}{1 + \sum_{j \in S} v_j}.
\end{equation*}
\end{definition}  

The MNL-bandit problem was initially motivated by fast fashion retailing and online advertising, and finds many applications in online learning, recommendation systems, and operations research (see \cite{avadhanula2019} for an overview). For instance, in fast fashion retailing, each item corresponds to a product and its reward is simply the revenue generated by selling the product. The assumption that $v_0 \ge \max\{v_1, \ldots, v_N\}$ can be justified by the fact that most customers do not buy anything in a shop visit. A similar phenomenon is also observed in online advertising where it is most likely that a user does not click any of the ads on a webpage when browsing. We naturally want to select a set of products/ads $S \subseteq [N]$ to display in the shop/webpage so that $R(S, \v)$, which corresponds to revenue generated by customer/user per visit, is maximized. 

We further pose a capacity constraint $K$ on the cardinality of $S$, since in most applications the size of the assortment cannot exceed a certain size. For example, the number of products presented at a retail shop is capped due to shelf space constraints, and the number of ads placed on a webpage cannot exceed a certain threshold.

In the MNL-bandit model, we need to {\em simultaneously} learn the item preference vector $\v$ and find the assortment with the maximum expected reward under $\v$. We approach this by repeatedly selecting an assortment to present to the user, observing the user's choice, and then trying to update the assortment selection policy. We call each observation of the user choice given an assortment a {\em pull}. We are interested in minimizing the number of pulls, which is the most expensive part of the learning process.

In bandit theory we are interested in two objectives. The first is called {\em regret minimization}: given a pull budget $T$, try to minimize the accumulated difference (called {\em regret}) between the sum of expected rewards of the optimal strategy in the $T$ pulls and that of the proposed learning algorithm; in the optimal strategy we always present the best assortment (i.e., the assortment with the maximum expected reward) to the user at each pull. The second is called {\em pure exploration}, where the goal is simply to identify the best assortment. 

Regret minimization in MNL-bandit has been studied extensively in the literature~\cite{RSS10,SZ13,DGT13,AAGZ16,AAGZ17,CW18}. The algorithms proposed in \cite{RSS10, SZ13} for the regret minimization problem make use of an ``exploration then exploitation'' strategy, that is, they first try to find the best assortment and then stick to it. However, they need the prior knowledge of the {\em gap} between the expected reward of the optimal assortment and that of the second-best assortment, which, in our opinion, is {\em unrealistic} in practice since the preference vector $\v$ is unknown at the beginning. We will give a more detailed discussion on these works in Section~\ref{sec:related}. 

In this paper we focus on pure exploration. Pure exploration is useful in many applications. For example, the retailer may want to perform a set of customer preference tests (e.g., crowdsourcing) to select a good assortment before the actual store deployment. We propose algorithms for pure exploration in MNL-bandit {\em without} any prior knowledge of preference vector. Our algorithms achieve {\em instance-sensitive} pull complexities which we elaborate next.

\vspace{2mm}
\noindent{\bf Instance Complexity.\ }
Before presenting our results, we give a few definitions and introduce {\em instance complexities} for pure exploration in MNL-bandit.

\begin{definition}[best assortment $S_\v$ and optimal expected reward $\theta_\v$]
\label{def:Sv}
Given a capacity parameter $K$ and a vector of item preferences $\v$, let 
\begin{equation*}
S_\v \triangleq \arg \max_{S \subseteq [N]: \abs{S} \le K} R(S, \v)
\end{equation*}
denote the best assortment with respect to $\v$. If the solution is not unique then we choose the one with the smallest cardinality which is unique (see the discussion after Lemma~\ref{lem:opt}). Let $\theta_\v \triangleq R(S_\v, \v)$ be the optimal expected reward.
\end{definition}

Denote $\eta_i \triangleq (r_i - \theta_{\v}) v_i$; we call $\eta_i$ the {\em advantage} of item $i$.  Suppose we have sorted the $N$ items according to $\eta_i$, let $\eta^{(j)}$ be the $j$-th largest value in the sorted list. 

\begin{definition}[reward gap $\Delta_i$]
\label{def:gap}
%Let $\alpha$ be the $K$-th largest value in $\{\eta_1, \ldots, \eta_N\}$ if $\abs{S_\v} = K$, or $\alpha = 0$ if $\abs{S_\v} < K$.
For any item $i \in [N] \backslash S_\v$, we define its reward gap to be
\begin{equation*}
\label{eq:h-1}
\Delta_i \triangleq  \left\{
  \begin{array}{ll}
   \eta^{(K)} - \eta_i, & \text{if} \abs{S_\v} = K,\\
    - \eta_i, & \text{if} \abs{S_\v} < K.
  \end{array}
  \right.
\end{equation*}
and for any item $i \in S_\v$, we define 
\begin{equation}
\label{eq:h-2}
\Delta_i \triangleq \bar{\Delta} = \min\left\{\left(\eta^{(K)} - \eta^{(K+1)}\right),\min_{j \in S_{\v}} \{r_j - \theta_{\v}\} \right\}.
\end{equation}
\end{definition}

%\nicksays{add more words on intuition behind $\Delta_i$? }
Definition~\ref{def:gap} may look a bit cumbersome. The extra term $\min_{j \in S_{\v}} \{r_j - \theta_{\v}\}$ in \eqref{eq:h-2} is added for a technical reason when handling the case that $\abs{S_\v} < K$; we will discuss this in more detail in Remark~\ref{rem:extra}. If $\abs{S_\v} = K$, then the task of finding the best assortment is equivalent to the task of identifying the $K$ items with the largest advantage values $\eta_i$, and the reward gap in Definition~\ref{def:gap} can be simplified as
\begin{eqnarray*}
\begin{array}{l}
\Delta_i =  \left\{
  \begin{array}{ll}
   \eta^{(K)} - \eta_i, & \forall i \in [N] \backslash S_\v,\\
   \eta^{(K)} - \eta^{(K+1)}, & \forall i \in S_\v.
  \end{array}
  \right.
\end{array}
\end{eqnarray*}
%This form is similar to the definition of the {\em mean gap} in the {\em multi-armed bandits} (MAB) model as we shall discuss shortly.

We now give two instance complexities for pure exploration in MNL-bandit.  The second can be viewed as a refinement of the first.
\begin{definition}[instance complexity $H_1$]
\label{def:H1}
We define the first instance complexity for pure exploration in MNL-bandit to be
$$H_1 \triangleq \sum\nolimits_{i \in [N]} \frac{1}{\Delta_i^2}.$$
\end{definition}
In this paper we assume that $\forall i \in [N], \Delta_i \neq 0$, since otherwise the complexity $H_1$ will be infinity.  This assumption implies that the best assortment is unique, which is also an essential  assumption for  works of literature whose pull complexities are based on ``assortment-level'' gaps, as we will discuss in Section~\ref{sec:related}.

Definition~\ref{def:H1} bears some similarity to the instance complexity defined for pure exploration in the {\em multi-armed bandits} (MAB) model, where we have $N$ items each of which is associated with an unknown distribution, and the goal is to identify the item whose distribution has the largest mean. 
In MAB the instance complexity is defined to be $H_{\tt MAB} = \sum_{i=2}^N 1/\Delta^2_i$, where $\Delta_i = \mu^{(1)} - \mu^{(i)}$ where $\mu^{(1)}$ is  the largest mean of the $N$ items and $\mu^{(i)}$ is the $i$-th largest mean of the $N$ items \cite{ABM10}.  Our definition of $\Delta_i$ is more involved due to the more complicated combinatorial structure of the MNL-bandit model.

\begin{definition}[instance complexity $H_2$]
\label{def:H2}
$$H_2 \triangleq \sum\nolimits_{i \in [N]} \frac{v_i + 1/K}{\Delta_i^2} +\max_{i \in [N]} \frac{1}{\Delta_i^2}.$$
\end{definition}
It is easy to see that $H_2 = O(H_1)$ (more precisely, $\frac{H_1}{K} \le H_2 \le 3 H_1$).  We comment that the $\max_{i \in [N]} \frac{1}{\Delta_i^2}$ term is needed only when $\abs{S_\v} < K$. 

\vspace{2mm}
\noindent{\bf Our Results.\ }  We propose two fixed-confidence algorithms for pure exploration in MNL-bandit.  The first one (Algorithm~\ref{alg:basic} in Section~\ref{sec:basic}) gives a pull complexity of $O\left(K^2 H_1 \ln\left({ \frac{N}{\delta} \ln (KH_1)}\right)\right)$ where $\delta$ is the confidence parameter.  We then modify the algorithm using a more efficient preference exploration procedure at each pull, and improve the asymptotic pull complexity to $O\left(K^2 H_2  \ln\left({ \frac{N}{\delta} \ln (KH_2)}\right)\right)$. The second algorithm is presented in Algorithm~\ref{alg:improve} in Section~\ref{sec:improve}.  

Both algorithms can be implemented efficiently: the time complexity of Algorithm~\ref{alg:basic} is bounded by $\tilde{O}(T + N^2)$ where $T$ is the pull complexity and `$\tilde{O}()$' hides some logarithmic factors. That of Algorithm~\ref{alg:improve} is bounded by $\tilde{O}(TN + N^2)$.\footnote{When we talk about {\em time complexity}, we only count the running time of the algorithm itself, and do not include the time for obtaining the pull results which depends on users' response time.}

As we shall discuss in Remark~\ref{rem:batch}, though having a larger pull complexity, Algorithm~\ref{alg:basic} still has the advantage that it better fits the {\em batched} model where we try to minimize the number of changes of the learning policy.

To complement our upper bounds, we prove that $\Omega(H_2/K^2)$ pulls is needed in order to identify the best assortment with probability at least $0.6$. This is presented in Section~\ref{sec:lb}.  Note that when $K$ is a constant, our upper and lower bounds match up to a logarithmic factor.

%We have implemented both algorithms and conducted experimental studies on both synthetic and real datasets.  The empirical results are consistent with our theoretical predictions, and have demonstrated the effectiveness of the proposed algorithms.

\subsection{Related Work.}
\label{sec:related}

Regret minimization in MNL-bandit was first studied by Rusmevichientong et al.~\cite{RSS10} in the setting of {\em dynamic assortment selection} under the MNL choice model.  Since then there have been a number of follow-ups that further improve the regret bound and/or remove some artificial assumptions~\cite{SZ13,DGT13,AAGZ16,AAGZ17,CW18}.  

As mentioned previously, the algorithms in \cite{RSS10,SZ13} also have a component of  identifying the best assortment.  In \cite{RSS10,SZ13}, the following ``assortment-level'' gap was introduced:
$$\Delta_{\tt asso} = \theta_\v - \max_{S \subseteq [N], \abs{S} \le K, S \neq S_\v} R(S, \v),$$ 
that is, the difference between the reward of the best assortment and that of the second-best assortment. The pull complexity of the component in \cite{SZ13} for finding the best assortment can be written as $\tilde{O}(KN/\Delta_{\tt asso}^2)$, where`$\tilde{O}()$' hides some logarithmic factors. This result is better than that in \cite{RSS10}.  There are two critical differences between these results and our results: 
(1) More critically, in \cite{RSS10,SZ13} it is assumed that the ``assortment-level'' gap $\Delta_{\tt asso}$ is known at the beginning, which is {\em not} practical since the fact that the preference vector is unknown at the beginning is a key feature of the MNL-bandit problem. 
(2) Our reward gaps $\Delta_i$ are defined at the ``item-level''; the instance complexity $H_1$ (or $H_2$) is defined as the sum of the inverse square of these item-level gaps and the total pull complexity is $\tilde{O}(K^2 H_1)$ (or $\tilde{O}(K^2 H_2)$). Though the two complexities are not directly comparable, the following example shows that for certain input instances, our pull complexity is significantly better. 
\begin{example}
$K = 1, r_1 = \ldots = r_N = 1, v_1 = 1, v_2 = 1-1/\sqrt{N}, v_3 = \ldots = v_N = 1/\sqrt{N}$.  We have $KN/\Delta_{asso} = \Omega(N^2)$, while $K^2 H_1 = O(N)$.  Thus, the pull complexity of the algorithm in \cite{RSS10} is {\em quadratic} of ours (up to logarithmic factors).
\end{example}

The best assortment identification problem has also been studied in the static setting (e.g., \cite{TR04,AVJ14}), where the user preference vector $\v$ is known as a priori and our task is to conduct an offline computation to find the assortment that maximizes the expected reward.  We refer readers to \cite{KF07} for an overview of this setting.

Chen et al.\cite{CLM18} studied the problem of top-$k$ ranking under the MNL choice model (but without the ``no purchase'' option). Their problem is different from ours: They aimed to find the $k$ items in $[N]$ with the largest preference $v_i$ (instead of the advantage $\eta_i = (r_i - \theta_\v) v_i$). In some sense, their problem can be thought of as a special case of ours, where $r_1 = r_2 = \ldots = r_N$ (that is, the rewards of all items are the same); but in their model, there is {\em no} null item.  It seems difficult to extend their approach to our  setting. 
%The algorithm in \cite{CLM18} is based on building a graph representation of the items which is entirely different from ours. 
We would also like to mention the work on battling-bandits by Saha and Gopalan~\cite{SG18}, who considered the problem of using the MNL choice model as one of the natural models to draw a winner from a set of items.  But their problem settings and the notion of the optimal solution are again different from the problem we consider here. 

Pure exploration has been studied extensively in the model of  MAB~\cite{EMM02,MT04,ABM10,GGLB11,GGL12,KKS13,JMNB14,KCG16,GK16,Russo16,CLQ17}.
MNL-bandit can be viewed as an MAB-type model with $\sum_{j \in [K]}{N \choose j}$ items, each corresponding to an assortment $S \subseteq [N]$ with $\abs{S}\le K$. However, these items may ``intersect'' with each other since assortments may contain the same items. Due to such dependencies, the algorithms designed for pure exploration in the MAB model cannot be adopted to the MNL-bandit model.  
Audibert et al.~\cite{ABM10} designed an instance-sensitive algorithm for the pure exploration problem in the MAB model. The result in \cite{ABM10} was later improved by Karnin et al.~\cite{KKS13} and Chen et al.\cite{CLQ17}, and extended into the problem of identifying multiple items~\cite{BWV13,ZCL14,CCZZ17}. 

Finally, we note that recently, concurrent and independent of our work, Yang~\cite{Yang21} has also studied pure exploration in MNL-bandit. But the definition of instance complexity in \cite{Yang21} is again at the ``assortment-level'' (and thus the results are not directly comparable), and the algorithmic approaches in \cite{Yang21} are also different from ours.  The pull complexity of \cite{Yang21} can be written as $\tilde{O}(H_{yang})$ where $H_{yang} = \sum_{i \in [N]} \frac{1}{(\Delta'_i)^2}$, where $\Delta'_i$ is defined to be the difference between the best reward among assortments that include item $i$ and that among assortments that exclude item $i$.  The following example shows that for certain input instances, our pull complexity based on item-level gaps is better.
\begin{example}
$r_1 = \ldots = r_N = 1, v_1 = \ldots = v_K = 1, v_{K+1} = \ldots = v_N = \eps$.  For $\eps \in (0, 1/K)$ and $\omega(1) \le K \le o(N)$, we have $H_{yang} = \Theta(N K^4)$, while our $K^2 H_2 = \Theta(K^5 + N K^3) = o(H_{yang})$.  
\end{example}

\section{Preliminaries}
\label{sec:preliminary}
Before presenting our algorithms, we would like to introduce some tools in probability theory and give some basic properties of the MNL-bandit model.  Due to space constraints, we leave the tools in probability theory (including Hoeffding's inequality, concentration results for the sum of geometric random variables, etc.) to Appendix~\ref{app:tool}.

%\subsection{Properties of the MNL-bandit Model}
%\label{sec:property}

The following (folklore) observation gives an effective way to check whether the expected reward of $S$ with respect to $\v$ is at least $\theta$ for a given value $\theta$. The proof can be found in Appendix~\ref{app:proof-ob-reward}.

\begin{observation}
\label{ob:reward}
For any $\theta \in [0, 1]$, $R(S, \v) \ge \theta$ if and only if $\sum_{i \in S} (r_i - \theta) v_i \ge \theta$.
\end{observation}

With Observation~\ref{ob:reward}, to check whether the maximum expected reward is at least $\theta$ for a given value $\theta$, we only need to check whether the expected reward of the particular set $S \subseteq [N]$ containing the up to $K$ items with the largest {\em positive} values $(r_i - \theta) v_i$ is at least $\theta$.  

To facilitate the future discussion we introduce the following definition.

\begin{definition}[$\Top(I, \v, \theta)$]
\label{def:top}
Given a set of items $I$ where the $i$-th item has reward $r_i$ and  preference $v_i$, and a value $\theta$, let $T$ be the set of $\min\{K, \abs{I}\}$ items with the largest values $(r_i - \theta)v_i$. Define
\(
\Top(I, \v, \theta) \triangleq T \setminus \{i \in I \mid (r_i - \theta) \le 0 \},
\)
where $\v$ stands for $(v_1, \ldots, v_{\abs{I}})$.
\end{definition}

The following lemma shows that $\Top(I, \v, \theta_\v)$ is exactly the best assortment.  Its proof can be found in Appendix~\ref{app:proof-lem-opt}.
\begin{lemma}\label{lem:opt}
	$\Top(I, \v, \theta_{\v}) = S_{\v}$.
\end{lemma}
Note that the set $\Top(I, \v, \theta_{\v})$ is unique by its definition. Therefore by Lemma~\ref{lem:opt} the set $S_\v$ is also uniquely defined.  

%By the definition of $\Top(I, \v, \theta_\v)$ we have the following.
%\begin{corollary}\label{cor:reward}
%	$\forall{i \in S_{\v}} : r_i > \theta_\v = R(S_\v, \v )$.
%\end{corollary}

We next show a monotonicity property of the expected reward function $R(\cdot,\cdot)$. Given two vectors $\v, \w$ of the same dimension, we write $\v \preceq \w$ if $\forall i, v_i \le w_i$.  We comment that similar properties appeared in \cite{AAGZ16, AAGZ17}, but were formulated a bit differently from ours.  The proof of Lemma~\ref{lem:mono} can be found in Appendix~\ref{app:proof-lem-mono}.

\begin{lemma}
\label{lem:mono}
	If $\v \preceq \w$, then $(\theta_\v =) R(S_\v, \v) \le (\theta_\w =) R(S_\w, \w)$, and for any $S \subseteq I$ it holds that	\[R(S, \w) - R(S, \v) \le \sum_{i \in S} (w_i - v_i).\]
\end{lemma}

The following is an immediate corollary of Lemma~\ref{lem:mono}.
\begin{corollary}
\label{cor:mono}
	If $\forall{i} : v_i \le w_i \le v_i + \frac{\epsilon}{K}$, then $\theta_{\v} \le \theta_{\w} \le \theta_{\v} + \epsilon$.
\end{corollary}

\section{The Basic Algorithm}
\label{sec:basic}

In this section, we present our first algorithm for pure exploration in MNL-bandit. The main algorithm is described in Algorithm~\ref{alg:basic}, which calls $\Prune$ (Algorithm~\ref{alg:prune}) and $\Explore$ (Algorithm~\ref{alg:explore}) as subroutines. $\Explore$ describes a pull of the assortment consisting of a single item.

\begin{algorithm}[t]

\caption{$\Explore(i)$}\label{alg:explore}
	\KwIn{Item $i$.}
	\KwOut{$0/1$ (choose or not choose $i$).}
	Offer a singleton set $S_i \gets \{i\}$ and observe a feedback $a$\;
	\lIf{$a = 0$}{\Return $1$}{\Return $0$}
\end{algorithm}

\begin{algorithm}[t]

\caption{$\Prune(I, K, \a, \b)$}\label{alg:prune}
	\KwIn{a set of items $I = \{1, \ldots, N\}$, capacity parameter $K$, two vectors $\a = (a_1, \ldots, a_N), \b = (b_1, \ldots, b_N)$ such that for any $i \in [N]$ it holds that $a_i \le v_i \le b_i$, where $\v = (v_1, \ldots, v_N)$ is the (unknown) preference vector of the $N$ items.}
	\KwOut{a set of candidate items for constructing the best assortment.}
	$\theta_{\a} \gets \max\limits_{S \subseteq I : \abs{S} \le K} R(S, \a)$,
	$\theta_{\b} \gets \max\limits_{S \subseteq I:  \abs{S} \le K} R(S, \b)$\;
	$C \gets \emptyset$\;
	\ForEach{$i \in I$}{
		form a vector $\g = (g_1, \ldots, g_N)$ s.t.\ $g_j = a_j$ for $j \neq i$, and $g_i = b_i$\label{ln:b-1}\;
		\lIf{$\exists \theta \in [\theta_{\a}, \theta_{\b}]$ s.t.\ $i \in \Top(I, \g, \theta)$}{add $i$ to $C$\label{ln:b-2}}
	}
	\Return{$C$}
\end{algorithm}

\begin{algorithm}[t]
\caption{The Fixed Confidence Algorithm for MNL-Bandit}
\label{alg:basic}
	\KwIn{a set of items $I = \{1, \ldots, N\}$, a capacity parameter $K$, a confidence parameter $\delta$.}
	\KwOut{the best assortment.}
	$I_0 \gets I$\;
	set $\epsilon_{\tau}= 2^{-\tau-3}$ for $\tau \ge 0$\;
	set $T_{-1} \gets 0$ and $T_{\tau} \gets \left\lceil \frac{32}{\epsilon^2_\tau} \ln{\frac{16 N (\tau + 1)^2}{\delta}} \right\rceil$ for $\tau \ge 0$\label{ln:c-0}\;
	\For{$\tau = 0, 1, \dotsc$}{
		\lForEach{$i \in I_\tau$}{call $\Explore(i)$ for $(T_{\tau} - T_{\tau - 1})$ times} 
		let $x^{(\tau)}_i$ be the mean of the outputs of the $T_\tau$ calls of $\Explore(i)$\;
		\ForEach{$i \in I_{\tau}$}{set $v^{(\tau)}_i \gets \min\{\frac{1}{x^{(\tau)}_i} - 1, 1\}$, $a^{(\tau)}_i \gets \max\{v^{(\tau)}_i - \epsilon_{\tau}, 0\}$, and $b^{(\tau)}_i \gets \min\{v^{(\tau)}_i + \epsilon_{\tau}, 1\}$\label{ln:c-1}\;}
		let $\a^{(\tau)}$ be the vector containing the $\abs{I_\tau}$ estimated preferences $a_i^{(\tau)}$,  and $\b^{(\tau)}$ be the vector containing the $\abs{I_\tau}$ estimated preferences $b_i^{(\tau)}$\label{ln:c-15}\;
		$C \gets \Prune(I_{\tau}, \a^{(\tau)}, \b^{(\tau)})$\label{ln:c-18}\;
		\If{$\left(|C| \le K\right) \land \left(\bigwedge\limits_{i \in C} \left(r_i > R\left(C, \b^{(\tau)}\right)\right)\right)$\label{ln:c-2}}{
			\Return{$C$\label{ln:c-3}} \;
		}
		$I_{\tau + 1} \gets C$\;
	}
\end{algorithm}

Let us describe the Algorithm~\ref{alg:prune} and \ref{alg:basic} in more detail.  Algorithm~\ref{alg:basic} proceeds in rounds.  In round $\tau$, each ``surviving'' item in the set $I_\tau$ has been pulled by $T_\tau$ times in total.  We try to construct two vectors $\a$ and $\b$ based on the empirical means of the items in $I_\tau$ such that the (unknown) true preference vector $\v$ of $I_\tau$ is tightly sandwiched by $\a$ and $\b$ (Line~\ref{ln:c-1}-\ref{ln:c-15}).  We then feed $I_\tau$, $\a$, and $\b$ to the $\Prune$ subroutine which reduces the size of $I_\tau$ by removing items that have no chance to be included in the best assortment (Line~\ref{ln:c-18}).  Finally, we test whether the output of $\Prune$ is indeed the best assortment (Line~\ref{ln:c-2}). If not we proceed to the next round, otherwise we  return the solution.

Now we turn to the $\Prune$ subroutine (Algorithm~\ref{alg:prune}), which is the most interesting part of the algorithm.  Recall that the two vectors $\a$ and $\b$ are constructed such that $\a \preceq \v \preceq \b$.  We try to prune items in $I$ by the following test: For each $i \in I$, we form another vector $\g$ such that $\g = \a$ in all coordinates except the $i$-th coordinate where $g_i = b_i$ (Line~\ref{ln:b-1}). We then check whether there exists a value $\theta \in [\theta_\a, \theta_\b]$ such that $i \in \Top(I, \g, \theta)$, where $\theta_\a, \theta_\b$ are the maximum expected rewards with $\a$ and $\b$ as the item preference vectors respectively; if the answer is Yes then item $i$ survives, otherwise it is pruned (Line~\ref{ln:b-2}). Note that our test is fairly conservative: we try to put item $i$ in a more favorable position by using the upper bound $b_i$ as its preference, while for other items we use the lower bounds $a_j$ as their preferences.  Such a conservative pruning step makes sure that the output $C$ of the $\Prune$ subroutine is always a superset of the best assortment $S_\v$.

\begin{theorem}
\label{thm:basic}
For any confidence parameter $\delta > 0$, Algorithm~\ref{alg:basic} returns the best assortment with probability $(1 - \delta)$ using at most $\Gamma = O\left(K^2 H_1 \ln\left({ \frac{N}{\delta} \ln (KH_1)}\right)\right)$ pulls.  The running time  of Algorithm~\ref{alg:basic} is bounded by $O\left(N\Gamma + N^2 \ln N \ln\left(\frac{K}{\min_{i \in I} \Delta_i}\right) \right)$.
\end{theorem}

In the rest of this section we prove Theorem~\ref{thm:basic}.  
%The proof is divided into three parts: the correctness, the pull complexity, and the time complexity.

\vspace{2mm}
\noindent{\bf Correctness.} We start by introducing the following event which we will condition on in the rest of the proof.  The event states that in any round $\tau$, the estimated preference $v^{(\tau)}_i$ for each item $i$ (computed at Line~\ref{ln:c-1} of Algorithm~\ref{alg:basic}) is at most $\epsilon_{\tau} = 2^{-\tau-3}$ away from the true preference $v_i$.
\begin{equation*}
\label{eq:E}
	\E_1 \triangleq \{\forall {\tau \ge 0}, \forall{i \in I_{\tau}} : \abs{v^{(\tau)}_i - v_i} < \epsilon_{\tau}\}.
\end{equation*}

The proof of the following lemma can be found in Appendix~\ref{app:proof-lem-E1}.  This lemma states that event $\E_1$ holds with high probability. 
\begin{lemma}
\label{lem:E1}
$\Pr[\E_1] \ge 1 - \delta$.
\end{lemma}

It is easy to see from Line~\ref{ln:c-1} of Algorithm~\ref{alg:basic} that conditioned on $\E_1$, we have
\begin{equation}
\label{eq:c-1}
	\forall \tau \ge 0: \a^{(\tau)} \preceq \v^{(\tau)} \preceq \b^{(\tau)},
\end{equation}
where $\v^{(\tau)}$ is the preference vector of items in $I_\tau$.

The following lemma shows that if (\ref{eq:c-1}) holds, then the $\Prune$ subroutine (Algorithm~\ref{alg:prune}) always produces a set of candidate items $C$ which is a superset of the best assortment.

\begin{lemma}
\label{lem:prune}
If the preference vector $\v$ of $I$ satisfies $\a \preceq \v \preceq \b$, then $\Prune(I, K, \a, \b)$ (Algorithm~\ref{alg:prune}) returns a set $C$ such that $S_\v \subseteq C$.
\end{lemma}

\begin{proof}
First, if $\a \preceq \v \preceq \b$, then by Lemma~\ref{lem:mono} we have $\theta_\v \in [\theta_\a, \theta_\b]$.  

Consider any item $i \in S_{\v}$, by the construction of $\g$ (Line~\ref{ln:b-1} of Algorithm~\ref{alg:prune}) we have for every $j \in I$: 
\begin{itemize}
\item if $j \neq i$, then $(r_j - \theta_{\v})g_j \le \max\{(r_j - \theta_{\v})v_j, 0\}$;

\item if $j = i$, then $(r_j - \theta_{\v})g_j \ge (r_j - \theta_{\v})v_j$.
\end{itemize}
By these two facts and the definition of $\Top(I, \v, \theta_{\v})$, we know that if $i \in \Top(I, \v, \theta_{\v})$, then $i \in \Top(I, \g, \theta_{\v})$.  Therefore for the particular value $\theta = \theta_\v  \in [\theta_\a, \theta_\b]$ we have $i \in \Top(I, \g, \theta)$, and consequently $i$ will be added to the candidate set $C$ at Line~\ref{ln:b-2}, implying that $S_\v \subseteq C$.
\end{proof}

Now suppose Algorithm~\ref{alg:basic} stops after round $\tau$ and outputs a set $C \supseteq S_\v$ of size at most $K$ (Line~\ref{ln:c-2}-\ref{ln:c-3}), then for any $i \in C$, we have $r_i > \theta_\b$. By Lemma~\ref{lem:mono} we also have $\theta_\b \ge \theta_\v$ (since $\v \preceq \b$). We thus have $r_i > \theta_\v$. Consequently, it holds that for every $i \in C$, $(r_i - \theta_\v) > 0$.  We thus have $C = S_\v$.

Up to this point we have shown that conditioned on $\E_1$, if Algorithm~\ref{alg:basic} stops, then it outputs the best assortment $S_\v$. We next bound the number of pulls the algorithm uses.

\vspace{2mm}
\noindent{\bf Pull Complexity.} We again conditioned on event $\E_1$. The next lemma essentially states that an item $i \in I \backslash S_\v$ will be pruned if its reward gap $\Delta_i$ is much larger than $K$ times its preference estimation error $\max\{b_i - v_i, v_i - a_i\}$.  

%The proof makes use of Corollary~\ref{cor:mono} and can be found in Appendix~\ref{app:proof-lem-prune-2}.

\begin{lemma}
\label{lem:prune-2}
In $\Prune(I, K, \a, \b)$ (Algorithm~\ref{alg:prune}), if $\a \preceq \v \preceq \b$, and $\forall i \in I: \max\{b_i - v_i, v_i - a_i\} \le \eps/K$ for any $\eps \in (0, 1)$, then any item $i \in I \backslash S_\v$ satisfying $\Delta_i > 8 \epsilon$ will not be added to set $C$.
\end{lemma}

\begin{proof}
	By Corollary~\ref{cor:mono}, if $\a \preceq \v \preceq \b$, and $\forall i \in I: \max\{b_i - v_i, v_i - a_i\} \le \eps/K$, then we have 
	\begin{equation}
	\label{eq:e-2}
	\theta_{\v} - \epsilon \le \theta_{\a} \le \theta_{\v} \le \theta_{\b} \le \theta_{\v} + \epsilon.
	\end{equation}
	Consider any item $i \in I \backslash S_\v$ with $\Delta_i > 8\eps$.  We analyze in two cases.
	
	{\bf Case 1:}  $\theta_\v - r_i > 8 \eps$.  By (\ref{eq:e-2}) we have  $\theta_\a - r_i > 7 \eps$. Therefore, for any $\theta \in [\theta_\a, \theta_\b]$ we have $r_i < \theta_\a \le \theta$, and consequently $i \not\in \Top(I, \g, \theta)$ for any $\theta \in [\theta_\a, \theta_\b]$ by the definition of $\Top()$.
	
	{\bf Case 2:}  $\theta_\v - r_i \le 8 \eps$.  First, note that if $\abs{S_\v} < K$, then we have 
	\begin{equation*}
	\Delta_i = -(r_i - \theta_\v) v_i = (\theta_\v - r_i) v_i \le \theta_\v - r_i \le 8\eps,
	\end{equation*} 
	contradicting our assumption that $\Delta_i > 8\eps$.   We thus focus on the case that $\abs{S_\v} = K$. We analyze two subcases. 
	\begin{enumerate}
		\item {\em $\theta \in (r_i, 1]$.} In this case, by the definition of $\Top()$ and the fact that $r_i - \theta < 0$, we have $i \not\in \Top(I, \g, \theta)$.
		
		\item {\em $\theta \in [\theta_\a, \theta_\b] \cap [0, r_i]$.}  For any $j \in S_\v$, we have
		\begin{eqnarray*}
			&& (r_i - \theta)g_i - (r_j - \theta)g_j \\
			& = & (r_i - \theta)b_i - (r_j - \theta)a_j \\
			& \le & (r_i - \theta)(v_i + \eps) - (r_j - \theta)a_j \quad (\text{since $r_i \ge \theta$})\\
			& \le & (r_i - \theta)v_i - (r_j - \theta)a_j + \eps \\
			& \le & (r_i - \theta_{\v})v_i - (r_j - \theta_{\v})a_j + (1 + a_j + v_i)\eps \quad (\text{by (\ref{eq:e-2})}) \\
			& \le & (r_i - \theta_{\v})v_i - (r_j - \theta_{\v})(v_j - \eps) + 3\eps \quad (\text{since $r_j > \theta_{\v}$})\\
			& \le & (r_i - \theta_{\v})v_i - (r_j - \theta_{\v})v_j + 4\eps \\
			& \le & -\Delta_i + 4\eps \\
			& < & -4\epsilon.	\quad (\text{by the assumption } \Delta_i > 8\eps)
		\end{eqnarray*}
		We thus have that for any  $\theta \in [\theta_\a, \theta_\b] \cap [0, r_i]$, $(r_i - \theta)g_i < (r_j - \theta)g_j$ for any $j \in S_\v$, therefore $i \not\in \Top(I, \g, \theta)$ for any $\theta \in [\theta_\a, \theta_\b]$, and consequently $i \not\in C$.
	\end{enumerate}
\end{proof}

For any $i \in I$, we define 
\begin{equation}
\label{eq:d-0}
		\tau(i) \triangleq \min \left\{\tau \ge 0 : \epsilon_{\tau} \le \frac{\Delta_i}{32 K}\right\}.
\end{equation}

The next lemma shows that item $i$ will {\em not} appear in any set $I_\tau$ with $\tau > \tau(i)$, and thus will not be pulled further after round $\tau(i)$.  
%Its proof makes use of Lemma~\ref{lem:prune-2} and can be found in Appendix~\ref{app:proof-lem-stop}.

\begin{lemma}
\label{lem:stop}
In Algorithm~\ref{alg:basic}, for any item $i \in I$, we have $i \not\in I_\tau$ for any $\tau > \tau(i)$.
\end{lemma}

\begin{proof}
	For any $i \in I \backslash S_\v$, setting $\eps = {\Delta_i}/{16}$. By (\ref{eq:d-0}) we have that for any $j \in I_{\tau(i)}$ it holds that 
	\begin{equation}
	\label{eq:f-1}
	\max\left\{v_j - a^{(\tau(i))}_{j}, b^{(\tau(i))}_j - v_{j}\right\} \le \frac{\Delta_i}{16K} = \frac{\eps}{K}.
	\end{equation}  
	Moreover, we have, 
	\begin{equation}
	\label{eq:f-2}
	\Delta_i = 16\eps > 8\eps.
	\end{equation}
	By (\ref{eq:f-1}), (\ref{eq:f-2}) and Lemma~\ref{lem:prune-2}, we have $i \not\in I_{\tau(i)+1}$.

	We next consider items in $S_\v$.  Note that by Definition~\ref{def:gap}, all $i \in S_\v$ have the same reward gap:
	\begin{equation*}
	\Delta_i = \bar{\Delta} \triangleq \min\{\min_{j \in I \backslash S_{\v}} \{\Delta_j\},\min_{j \in S_{\v}} \{r_j - \theta_{\v}\} \} \le \min_{j \in  I \backslash S_{\v}} \{\Delta_j\}.
	\end{equation*} 
	Let 
	\begin{equation}
	\label{eq:bar-tau}
	\bar{\tau} \triangleq \min \left\{\tau \ge 0 : \epsilon_{\tau} \le \frac{\bar{\Delta}}{32 K}\right\}. 
	\end{equation} 
	We thus have $\bar{\tau} = \tau(i)$ for all $i \in S_\v$, and $\bar{\tau} \ge \tau(j)$ for any $j \in  I \backslash S_{\v}$.  Therefore, at the end of round $\bar{\tau}$, all items in $I \backslash S_\v$ have already been pruned, and consequently,
	\begin{equation}
	\label{eq:g-1}
	%\abs{I_{\bar{\tau}+1}} \le K.
	\abs{C} \le K.
	\end{equation}
	
	By (\ref{eq:d-0}) and Corollary~\ref{cor:mono} we have $\theta_{\b^{(\bar{\tau})}} \le \theta_\v + \bar{\Delta}/16$. Consequently we have
	\begin{eqnarray}
	r_i - R(C, \b^{(\bar{\tau})}) & = & r_i - \theta_{\b^{(\bar{\tau})}} = (r_i - \theta_{\v}) - (\theta_{\b^{(\bar{\tau})}} - \theta_{\v}) \nonumber \\
	& \ge & \bar{\Delta} - \frac{\bar{\Delta}}{16} > 0\,.
	\label{eq:g-2} 
	\end{eqnarray}
	By (\ref{eq:g-1}) and (\ref{eq:g-2}), we know that Algorithm~\ref{alg:basic} will stop after round $\bar{\tau}$ and return $C = S_\v$.
\end{proof}
\begin{comment}
Let us give some ideas on how to prove Lemma~\ref{lem:stop}. Note that at the end of the $\tau(i)$-th round, for any $j \in I_{\tau(i)}$, we have by event $\E_1$ and the definitions of vectors $\a^{(\tau(i))}_j$ and $\b^{(\tau(i))}_j$ that
\begin{eqnarray}
&& \max\left\{v_j - a^{(\tau(i))}_{j}, b^{(\tau(i))}_j - v_{j}\right\} \nonumber\\
&\le& \max\left\{\left(v_j - v^{(\tau(i))}_j\right) + \left(v^{(\tau(i))}_j - a^{(\tau(i))}_{j}\right),  \left(b^{(\tau(i))}_j - v^{(\tau(i))}_j\right) + \left(v^{(\tau(i))}_j - v_j\right)\right\} \nonumber\\
&\le& 2 \epsilon_{\tau(i)} \le \frac{\Delta_i}{16 K} \label{eq:d-1}. \quad (\text{by (\ref{eq:d-0})})
\end{eqnarray}
In words, (\ref{eq:d-1}) shows that the (unknown) preference vector $\v$ is sandwiched within a factor of $\Theta(\frac{\Delta_i}{K})$ by two estimation vectors $\a^{(\tau(i))}$ and $\b^{(\tau(i))}$ in round $\tau(i)$.  This fact will be used in the following lemma, which is the key for proving Lemma~\ref{lem:stop}. Lemma~\ref{lem:prune-2} essentially says that an item $i \in I \backslash S_\v$ will be pruned when the estimation gap $\frac{\Delta_i}{K}$ is small enough.
\end{comment}

%The proof of Lemma~\ref{lem:prune-2} can be found in Appendix~\ref{app:proof-lem:prune-2}.

With Lemma~\ref{lem:stop} we can easily bound the total number of pulls made by Algorithm~\ref{alg:basic}. By (\ref{eq:d-0}) we have $\tau(i) = O\left(\ln\left(\frac{K}{\Delta_i}\right)\right)$. By the definition of $T_\tau$ (Line~\ref{ln:c-0} of Algorithm~\ref{alg:basic}), the total number of pulls is at most
\begin{eqnarray*}
	\sum_{i \in I} T_{\tau(i)} &\le& O\left(\sum_{i \in I} \frac{K^2}{\Delta^2_i} \ln \frac{N \tau^2(i)}{\delta}\right) \\
	&=&  O\left(K^2 H_1  \ln\left({ \frac{N}{\delta} \ln (KH_1)}\right)\right).
\end{eqnarray*}

\begin{remark}
	\label{rem:extra}
	The reason that we introduce an extra term $\min_{j \in S_\v}\{r_j - \theta_\v\}$ in the definition of reward gap $\Delta_i$ for all $i \in S_\v$ (Definition~\ref{def:gap}) is for handling the case when $\abs{S_\v} < K$. More precisely, in the case $\abs{S_\v} < K$ we have to make sure that for all items $i \in I$ that we are going to add into the best assortment $S_\v$, it holds that $r_i > \theta_\v$. In our proof this is guaranteed by (\ref{eq:g-2}). 
%	where we have used the fact $\bar{\Delta} \le r_i - \theta_\v$ to make sure that $r_i \ge \theta_{\b^{\bar{\tau}}} > \theta_\v$.  
	On the other hand, if we are given the promise that $\abs{S_\v} = K$ (or $\abs{S_\v} = K'$ for a fixed value $K' \le K$), then we do not need this extra term: we know when to stop simply by monitoring the size of $I_\tau$, since at the end all items $i \in I/S_\v$ will be pruned.
\end{remark}

\vspace{2mm}
\noindent{\bf Running Time.}  
Finally, we analyze the time complexity of Algorithm~\ref{alg:basic}.  Although the time complexity of the algorithm is not the first consideration in the MNL-bandit model, we believe it is important for the algorithm to finish in a reasonable amount of time for real-time decision making.  Observe that the running time of Algorithm~\ref{alg:basic} is dominated by the sum of the total number of pulls and the running time of the $\Prune$ subroutine, which is the main object that we shall bound next. 

Let us analyze the running time of $\Prune$. Let $n \triangleq \abs{I}$. First, $\theta_\a$ and $\theta_\b$ can be computed in $O(n^2)$ time by an algorithm proposed by Rusmevichientong et al. \cite{RSS10}.  We next show that Line~\ref{ln:b-2} of Algorithm~\ref{alg:prune} can be implemented in $O(n \ln n)$ time, with which the total running time of $\Prune$ is bounded by $O(n^2 \ln n)$.

Consider any item $i \in I$.  We can restrict our search of possible $\theta$ in the range of $\Theta_i = [\theta_\a, \theta_\b] \cap [0, r_i)$, since if $i \in \Top(I, \g, \theta)$, then by the definition of $\Top()$ we have $\theta < r_i$. For each $j \neq i, j \in I$, define 
\begin{equation*}
\Theta_{j} = \{\theta \in \Theta_i \mid  (r_j - \theta)g_j > (r_i - \theta)g_i\}.
\end{equation*}
Intuitively speaking, $\Theta_j$ contains all $\theta$ values for which item $j$ is ``preferred to'' item $i$ for $\Top(I, \g, \theta)$. Consequently, for any $\theta \in \Theta_i$, if the number of $\Theta_j$ that contain $\theta$ is at least $K$, then we have $i \not\in \Top(I, \g, \theta)$; otherwise if the number of such $\Theta_j$ is less than $K$, then we have $i \in \Top(I, \g, \theta)$.  Note that each set $\Theta_j$ can be computed in $O(1)$ time.

Now think each set $\Theta_j$ as an interval.  The problem of testing whether there exists a $\theta \in [\theta_\a, \theta_\b] \cap [0, r_i)$ such that $i \in \Top(I, \g, \theta)$ can be reduced to the problem of checking whether there is a $\theta \in [\theta_\a, \theta_\b] \cap [0, r_i) $ such that $\theta$ is contained in fewer than $K$ intervals $\Theta_j\ (j \neq i)$. The later problem can be solved by the standard sweep line algorithm in $O(n \ln n)$ time.

Recall that the total number of rounds can be bounded by $\tau_{\max} = \max_{i \in I} \tau(i) = O\left(\ln\left(\frac{K}{\min_{i \in I} \Delta_i}\right)\right)$. Therefore the total running time of Algorithm~\ref{alg:basic} can be bounded by
{\small
\begin{equation*}
O\left(\Gamma +  \sum_{\tau = 0}^{\tau_{\max}} \abs{I_{\tau}}^2 \ln \abs{I_{\tau}} \right) = O\left(\Gamma + N^2 \ln N \ln\left(\frac{K}{\min_{i \in I} \Delta_i}\right) \right),
\end{equation*}
}
where $\Gamma = O\left(K^2 H_1 \ln\left({ \frac{N}{\delta} \ln (KH_1)}\right)\right)$ is the total number of pulls made by the algorithm.

\section{The Improved Algorithm}
\label{sec:improve}

In this section we try to improve our basic algorithm presented in Section~\ref{alg:basic}. We design an algorithm whose pull complexity depends on $H_2$ which is asymptotically at most $H_1$.  The improved algorithm is described in Algorithm~\ref{alg:improve}.  

The structure of Algorithm~\ref{alg:improve} is very similar to that of Algorithm~\ref{alg:basic}.  The main difference is that instead of using $\Explore$ to pull a singleton assortment at each time, we use a new procedure $\ExploreS$ (Algorithm~\ref{algo:explore-set}) which pulls an assortment of size up to $K$ (Line~\ref{ln:e-1} of Algorithm~\ref{alg:improve}).  We construct the assortments by partitioning the whole set of items $I_\tau$ into subsets of size up to $K$ (Line~\ref{ln:e-2}-\ref{ln:e-3}). In the $\ExploreS$ procedure, we keep pulling the assortment $S$ until the output is $0$ (i.e., a no-purchase decision is made).  We then estimate the preference of item $i$ using the average number of times that item $i$ is chosen in those \ExploreS\ calls that involve item $i$ (Line~\ref{ln:e-4}).

Intuitively, $\ExploreS$ has the advantage over $\Explore$ in that at each pull, the probability for $\ExploreS$ to return an item instead of a no-purchase decision is higher, and consequently $\ExploreS$ extracts more information about the item preferences.  We note that the $\ExploreS$ procedure was first introduced in \cite{AAGZ19} in the setting of regret minimization.

\begin{algorithm}[t]
	\caption{$\ExploreS(S)$}\label{algo:explore-set}
	\KwIn{a set of items $S$ of size at most $K$.}
	\KwOut{a set of empirical preferences $\{f_i\}_{i \in S}$.}
	Initialize $f_i \gets 0$ for $i \in S$\;
	\Repeat{$a = 0$}{
		offer assortment $S$ and observe a feedback $a$\;
		\lIf{$a \in S$}{
			$f_a \gets f_a + 1$
		}
	}
	\Return{$\{f_i\}_{i \in S}$}
\end{algorithm}

\begin{algorithm}[t]
	\caption{Improved Fixed Confidence Algorithm for MNL-bandit} 
	\label{alg:improve}
	\KwIn{a set of items $I = \{1, \dotsc, N\}$, a capacity parameter $K$, and a confidence parameter $\delta$.}
	\KwOut{the best assortment.}
	set $I_0 \gets I$, and $\epsilon_{\tau} = 2^{-\tau-3}$ for $\tau \ge 0$\;
	set $T_{-1} \gets 0$, and $T_{\tau} \gets \left\lceil \frac{8}{\epsilon^2_{\tau}} \ln{\frac{16 N (\tau + 1)^2}\delta}\right\rceil$ for $\tau \ge 0$\;
	\For{$\tau = 0, 1, \dotsc$}{
		$m_{\tau} \gets \lceil \abs{I_\tau}/{K}\rceil$\label{ln:e-2}\;
		let $S^{\tau}_1 \uplus \dotsc \uplus S^{\tau}_{m_{\tau}}$ be an arbitrary partition of $I_{\tau}$ into subsets of size at most $K$\label{ln:e-3}\;
		\lForEach{$j \in [m_{\tau}]$}{
			call $\ExploreS(S^{\tau}_j)$ for $(T_{\tau} - T_{\tau - 1})$ times \label{ln:e-1}
		}
		\ForEach{$i \in I_\tau$}{let $v^{(\tau)}_i$ be the average of $f_i$'s returned by the multiset of calls $\{\ExploreS(S^{\rho}_j)\ |\ \rho \le \tau, \ j \in [m_\rho], i \in S^{\rho}_j\}$\;\label{ln:e-4}}
		\lForEach{$i \in I_{\tau}$}{
			set $a^{(\tau)}_i \gets \max\{0, v^{(\tau)}_i - \epsilon_{\tau}\}$ and $b^{(\tau)}_i \gets \min\{v^{(\tau)}_i + \epsilon_{\tau}, 1\}$
		}
		let $\a^{(\tau)}$ be the vector containing the $\abs{I_\tau}$ estimated preferences $a_i^{(\tau)}$,  and $\b^{(\tau)}$ be the vector containing the $\abs{I_\tau}$ estimated preferences $b_i^{(\tau)}$\;
		$C \gets \Prune(I_{\tau}, \a^{(\tau)}, \b^{(\tau)})$\;
				\If{$\left(|C| \le K\right) \land \left(\bigwedge\limits_{i \in C} \left(r_i > R\left(C, \b^{(\tau)}\right)\right)\right)$}{
			\Return{$C$} \;
		}
		$I_{\tau + 1} \gets C$ \;
	}
\end{algorithm}

%In this section we prove the following theorem.
\begin{theorem}
\label{thm:improve}
For any confidence parameter $\delta > 0$, Algorithm~\ref{alg:improve} returns the best assortment with probability $(1 - \delta)$ using at most $\Gamma = O\left(K^2 H_2 \ln \left({ \frac{N}{\delta} \ln (KH_2)}\right)\right)$ pulls. The running time of Algorithm~\ref{alg:improve} is bounded by $O\left(N\Gamma + N^2 \ln N \ln\left(\frac{K}{\min_{i \in I} \Delta_i}\right) \right)$. 
\end{theorem}

Compared with Theorem~\ref{thm:basic}, the only difference in the pull complexity of Theorem~\ref{thm:improve} is that we have used $H_2$ instead of $H_1$.  Since $H_2 = O(H_1)$, the asymptotic pull complexity of Algorithm~\ref{alg:improve} is at least as good as that of Algorithm~\ref{alg:basic}.

\begin{remark}
\label{rem:batch}
Though having a higher pull complexity, Algorithm~\ref{alg:basic} still has an advantage against Algorithm~\ref{alg:improve} in that Algorithm~\ref{alg:basic} can be implemented in the batched setting with $\max_{i \in I}\tau(i) = O\left(\ln\frac{K}{\min_{i \in I}\Delta_i}\right)$ policy changes, which cannot be achieved by Algorithm~\ref{alg:improve} since the subroutine $\ExploreS$ is inherently sequential.
\end{remark}

Compared with the proof for Theorem~\ref{thm:basic}, the challenge for proving Theorem~\ref{thm:improve} is that the number of pulls in each $\ExploreS$ is a random variable. We thus need slightly more sophisticated mathematical tools to bound the sum of these random variables.  
%In the rest of this section we prove Theorem~\ref{thm:improve}.  
Due to the space constraints, we leave the technical proof of Theorem~\ref{thm:improve} to Appendix~\ref{app:proof-thm-improve}.

\section{Lower Bound}
\label{sec:lb}

We manage to show the following lower bound to complement our algorithmic results.  

\begin{theorem}
\label{thm:lb}
For any algorithm $\A$ for pure exploration in multinomial logit bandit, there exists an input instance such that $\A$ needs
$\Omega(H_2/K^2)$ pulls to identify the best assortment with probability at least $0.6$.
\end{theorem}

Note that Algorithm~\ref{alg:improve} identifies the best assortment with probability $0.99$ using at most $\tilde{O}(K^2 H_2)$ pulls (setting $\delta = 0.01$). Therefore our upper and lower bounds match up to a logarithmic factor if $K = O(1)$.

The proof of Theorem~\ref{thm:lb} bears some similarity with the lower bound proof of the paper by Chen et al.~\cite{CLM18}, but there are some notable differences. As mentioned in the introduction, Chen et al.~\cite{CLM18} considered the problem of top-$k$ ranking under the MNL choice model, which differs from the best assortment searching problem in the following aspects:
\begin{enumerate} 
\item The top-$k$ ranking problem can be thought as a special case of the best assortment searching problem where the rewards of all items are equal to $1$.  While to prove Theorem~\ref{thm:lb} we need to choose hard instances in which items have {\em different} rewards.

\item There is {\em no} null item (i.e., the option of ``no purchase'') in the top-$k$ ranking problem.  Note that we cannot treat the null item as the $(N+1)$-th item with reward $0$ since the null item will appear implicitly in every selected assortment.
\end{enumerate} 
These two aspects prevent us to use the lower bound result in Chen et al.~\cite{CLM18} as a blackbox, and some new ideas are needed for proving Theorem~\ref{thm:lb}.  Due to the space constraints, we leave the technical proof to Appendix~\ref{app:lb-proof}.

\section{Concluding Remarks}
\label{sec:conclude}

We would like to conclude the paper by making a few remarks.  First, our upper and lower bounds are almost tight only when $K = O(1)$. Obtaining tight bounds with respect to general $K$ remains to be an interesting open question. 

Second, our algorithms for pure exploration can also be used for regret minimization under the ``exploration then exploitation'' framework.  Setting $\delta = 1/T$, Algorithm~\ref{alg:basic} gives a regret of $O\left(K^2 H_1 \ln\left({ NT \ln (KH_1)}\right)\right)$, and Algorithm~\ref{alg:improve} gives a regret of $O\left(K^2 H_2 \ln \left({NT \ln (KH_2)}\right)\right)$. These bounds are pretty crude since we assume that each pull gives a regret of $1$. Again, these bounds are not directly comparable with those in the previous work due to our new definitions of instance complexities $H_1$ and $H_2$.

Third, our algorithms for pure exploration fall into the category of {\em fixed-confidence} algorithms, that is, for a fixed confidence parameter $\delta$, we want to identify the best assortment with probability at least $(1 - \delta)$ using the smallest number of pulls. Another variant of pure exploration is called {\em fixed-budget} algorithms, where given a fixed pull budget $T$, we try to identify the best assortment with the highest probability.  We leave this variant as future work.

\newpage

\bibliographystyle{plain}
\bibliography{paper,focs19}

\newpage 

\appendix

\noindent\rule{\textwidth/2}{1pt}
\begin{center}
	\textbf{\large Appendix for Instance-Sensitive Algorithms for Pure Exploration in Multinomial Logit Bandit}
\end{center}
\noindent\rule{\textwidth/2}{1pt}

\section{More Preliminaries}
\label{app:preliminary-supplementary}

\subsection{Tools in Probability Theory}
\label{app:tool}
We make use of the following standard concentration inequalities.

\begin{lemma}[Hoeffding’s inequality]
	\label{lem:chernoff}
	Let $X_1, \dotsc, X_n \in [0, 1]$ be independent random variables and $X = \sum\limits_{i = 1}^n X_i$.
	Then
	\[
	\Pr[X > \bE[X] + t] \leq \exp\left(-{2 t^2}/{n}\right)
	\]
	\text{and}
	\[
	\Pr[X < \bE[X] - t] \leq \exp\left(-{2 t^2}/{n}\right)
	\,.
	\]
\end{lemma}

\begin{lemma}[Azuma's inequality]
	\label{lem:azuma}
	Let the sequence $Z_0, \dotsc, Z_n$ be a submartingale and
	\[
	\forall{t \in [n]} : \abs{Z_t - Z_{t - 1}} \le d\,.
	\]
	Then 
	\[
	\Pr[Z_n - Z_0 \le -\eps] \le \exp\left(\frac{-\eps^2}{2 d^2 n}\right)\,.
	\]
\end{lemma}

\begin{definition}[geometric random variable; the failure model]
	\label{def:geo}
	Let $p \in [0,1]$. If a random variable $X$ with support $\mathbb{Z}^+$ satisfies $\Pr[X = k] = (1 - p)^{k} p$ for any integer $k \ge 0$, then we say $X$ follows the geometrical distribution with parameter $p$, denoted by $X \sim \Geo(p)$. 
\end{definition}

The following lemma gives the concentration result for sum of geometric random variables with a {\em multiplicative} error term.
\begin{lemma}[\cite{Janson18}]
	\label{lem:geo-multi}
	Let $p \ge 0$, $\lambda \ge 1$, and $X_1, \dotsc, X_n$	be i.i.d.\ random variables from distribution $\Geo(1 / (1 + p))$. We have
	\begin{equation}
	\Pr\left[\sum_{i = 1}^n (X_i + 1) \ge \lambda n (1 + p)\right] \le \exp(- n (\lambda - 1 - \ln \lambda))\,.
	\end{equation}
	%In particular, by setting $\lambda = 5$ we get $	\Pr\left[\sum_{i = 1}^n (X_i + 1) \ge 5 n (1 + q)\right] \le \exp(- 2 n)$.  \qinsays{add this here or later?}
\end{lemma}

In our analysis we need the following concentration result for sum of geometric random variables with an {\em additive} error term.

\begin{lemma}
	\label{lem:geo-additive}
	Let $p \in [0, 1]$, $t \in [0, 1]$, and $X_1, \ldots, X_n$ be i.i.d.\ random variables from distribution $\Geo(1 / (1 + p))$. We have
	\begin{equation}
	\Pr\left[\abs{\frac{1}{n} \sum_{i=1}^n (X_i - p)} \ge t\right] \le 2\exp\left(- \frac{n t^2}{8}\right).
	\end{equation}
\end{lemma}

\begin{proof}
We use the following lemma to derive Lemma~\ref{lem:geo-additive}.
\begin{lemma}[\cite{JLWZ19}]
\label{lem:geo-multi-2}
	Let $p > 0$ and $X_1, \dotsc, X_n$ be i.i.d. random variables from $\Geo(1 / (1 + p))$, then for $\lambda \in (0, 1]$ we have
	\begin{equation}
	\label{eq:geo-1}
	\Pr\left[\frac{1}{n}\sum_{i = 1}^n X_i \le \lambda p\right] \le \exp\left(-n \cdot \frac{p(\lambda - 1)^2}{2 (1 + p)}\right)\,
	\end{equation}
	for $\lambda \in [1, 2)$
	\begin{equation}
	\label{eq:geo-2}
	\Pr\left[\frac{1}{n}\sum_{i = 1}^n X_i \ge \lambda p\right] \le \exp\left(-n \cdot \frac{p(\lambda - 1)^2}{4 (1 + p)}\right)\,
	\end{equation}
	and for $\lambda \ge 2$ 
	\begin{equation}
	\label{eq:geo-3}
	\Pr\left[\frac{1}{n} \sum_{i = 1}^n X_i \ge \lambda p \right] \le \exp\left(-n \cdot \frac{p(\lambda - 1)}{4 (1 + p)}\right)\,.
	\end{equation}
\end{lemma}

Note that the lemma holds trivially for $p = 0$. We thus focus on the case $p > 0$.  We first show  
\begin{equation}
\label{eq:i-1}
\Pr\left[\frac{1}{n}\sum_{i = 1}^n (X_i - p) \le -t \right] \le \exp\left(- \frac{n t^2}{8}\right).
\end{equation} 
We analyze in two cases.
\begin{enumerate}
\item If $1 \ge t \ge p$, then we have 
%\begin{equation*}
%\Pr\left[\frac{1}{n}\sum_{i = 1}^n (X_i - p) \le -t\right] \le \Pr\left[\sum_{t=1}^n X_i < 0\right] = 0
%\le \exp\left(- \frac{n t^2}{8}\right).
%\end{equation*}

%\item If $t = p$, then we have
\begin{eqnarray*}
&&\Pr\left[\frac{1}{n}\sum_{i = 1}^n (X_i - p) \le -t\right] \\
&\le& \Pr\left[\sum_{t=1}^n X_i = 0\right] = \prod_{i = 1}^n \Pr[X_i = 0] \\
&=& \left({\frac{p}{1 + p}}\right)^n \le 2^{-n} 
\le \exp\left(- \frac{n t^2}{8}\right).
\end{eqnarray*}

\item If $t < p \le 1$, then by (\ref{eq:geo-1}), setting $\lambda = 1 - \frac{t}{p}$, we have
\begin{eqnarray*}
\Pr\left[\frac{1}{n}\sum_{i = 1}^n (X_i - p) \le -t\right] 
&\le&  \exp\left(- \frac{np\cdot(t/p)^2}{2 (1 + p)}\right) \\
&\le& \exp\left(- \frac{n t^2}{8}\right).
\end{eqnarray*}
\end{enumerate} 
We next show 
\begin{equation}
\label{eq:i-2}
\Pr\left[\frac{1}{n}\sum_{i = 1}^n (X_i - p) \ge t \right] \le \exp\left(- \frac{n t^2}{8}\right).
\end{equation} 
We analyze in two cases.
\begin{enumerate}
\item If $t \le p$, then by (\ref{eq:geo-2}), setting $\lambda = 1 + \frac{t}{p}$, we have
\begin{eqnarray*}
	\Pr\left[\frac{1}{n}\sum_{i = 1}^n (X_i - p) \ge t\right] &\le& \exp\left(-\frac{np \cdot (t/p)^2}{4 (1 + p)}\right) \\
	&\le& \exp\left(- \frac{n t^2}{8}\right).
\end{eqnarray*}
	
\item If $p < t \le 1$, then by (\ref{eq:geo-3}), setting $\lambda = 1 + \frac{t}{p}$, we have
\begin{eqnarray*}
	\Pr\left[\frac{1}{n}\sum_{i = 1}^n (X_i - p) \ge t\right] &\le& \exp\left(-\frac{np \cdot (t/p)^2}{4 (1 + p)}\right) \\
	&\le&  \exp\left(- \frac{n t^2}{8}\right).
\end{eqnarray*}
\end{enumerate}
\end{proof}

\subsection{Proof of Observation~\ref{ob:reward}}
\label{app:proof-ob-reward}
\begin{proof}
	The observation follows directly from the definition of expected reward (Definition~\ref{def:reward}). That is, $R(S, \v) \ge \theta$ means $\frac{\sum_{i \in S} r_i v_i}{1 + \sum_{i \in S}v_i} \ge \theta$, which implies $\sum_{i\in S} (r_i - \theta)v_i \ge \theta$. The other direction can be shown similarly.
\end{proof}

\subsection{Proof of Lemma~\ref{lem:opt}}
\label{app:proof-lem-opt}

The following is an easy observation by the definition of $\Top(I, \v, \theta)$.
\begin{observation}\label{ob:top}
	For any $S \subseteq I$ of size at most $K$ and any $\theta \in [0, 1]$, it holds that
	\begin{equation*}
	\sum_{i \in S}(r_i - \theta)v_i \le \sum_{i \in \Top(I, \v, \theta)}(r_i - \theta)v_i\,.
	\end{equation*}
\end{observation}

The following claim gives a crucial property of $\Top(I, \v, \theta)$.  Lemma~\ref{lem:opt} follows immediately from this claim.

\begin{claim}
\label{cla:top}
For any $\theta \in [0, 1]$,
$\theta \le \theta_\v$  if and only if $\sum\limits_{i \in \Top(I, \v, \theta)} (r_i - \theta)v_i \ge \theta$.
\end{claim}
\begin{proof}
	First consider the case $\theta \le \theta_\v$. By Observation~\ref{ob:reward}, $R(S_\v, \v) = \theta_\v \ge \theta$ implies $\sum_{i \in S_\v}(r_i - \theta)v_i \ge \theta$. Then by Observation~\ref{ob:top} we have $\sum_{i \in \Top(I, \v, \theta)} (r_i - \theta)v_i \ge \sum_{i \in S_\v}(r_i - \theta)v_i \ge \theta$.
	
	Next consider the case $\theta > \theta_\v$.  For any $S \subseteq I$ with $\abs{S} \le K$, by the definition of $\theta_\v$ we have $R(S, \v) \le \theta_\v$. Then by Observation~\ref{ob:reward} we have $\sum_{i \in S} (r_i - \theta) v_i \le \theta_\v < \theta$ for {\em any} $S \subseteq I$ with $\abs{S} \le K$.  Consequently, we have $\sum_{i \in \Top(I, \v, \theta)} (r_i - \theta) v_i < \theta$.
\end{proof}

\subsection{Proof of Lemma~\ref{lem:mono}}
\label{app:proof-lem-mono}

\begin{proof}
	If $\v \preceq \w$, then by definition of $S_{\w}$ we have
	\(
	R(S_\w, \w) \ge R(S_\v, \w)
	\), and
	
	\begin{eqnarray*}
		\sum_{i \in S_\v}(r_i - \theta_\v)w_i &\ge& \sum_{i \in S_\v}(r_i - \theta_\v)v_i   \quad \text{($\v \preceq \w$)} \\
		&\ge& \theta_\v \ \  \text{($R(S_\v, \v) \ge \theta_\v$ and Observation~\ref{ob:reward})} \\ 
		&=& R(S_\v, \v).
	\end{eqnarray*}
We thus have $R(S_{\w}, \w) \ge R(S_{\v}, \w) \ge R(S_{\v}, \v)$.
	
	The second part of the lemma is due to the following simple calculation. Recall that $r_i \in (0,1]$ for any $i \in S$.
	\begin{equation*}
	R(S, \w) - R(S, \v) \le \frac{\sum_{i \in S} r_i (w_i - v_i)}{1 + \sum_{i \in S}v_i} \le \sum_{i \in S} (w_i - v_i)\,.
	\end{equation*}
\end{proof}

%\section{Missing Proofs in Section~\ref{sec:basic}}
%\label{app:proofs-basic}

\section{Proof of Lemma~\ref{lem:E1}}
\label{app:proof-lem-E1}
\begin{proof}
	The output of $\Explore(i)$ is a Bernoulli random variable with mean $x_i = \frac{1}{1+v_i}$. By Hoeffding's inequality (Lemma~\ref{lem:chernoff}) we have
	\begin{equation*}
	\Pr\left[\abs{x^{(\tau)}_i - x_i} \ge \frac{\epsilon_{\tau}}{8}\right] \le 2\exp\left(-\frac{\epsilon^2_{\tau}T_{\tau}}{32} \right) \le \frac{\delta}{8 N (\tau + 1)^2}.
	\end{equation*}
	By a union bound we have 
	\begin{eqnarray}
	&& \Pr\left[\forall{\tau \ge 0}, \forall{i \in I_\tau} : \abs{x^{(\tau)}_i - x_i} < \frac{\epsilon_{\tau}}{8}\right] \nonumber \\
	&\ge& 1 - \sum_{\tau = 0}^{\infty} \sum_{i \in I_\tau} \frac{\delta}{8 N (\tau + 1)^2}  \ge 1 - \delta.
	\label{eq:b-2}
	\end{eqnarray}
	Since at Line~\ref{ln:c-1} of Algorithm~\ref{alg:basic} we have set $v^{(\tau)}_i = \frac{1}{x^{(\tau)}_i} - 1$, with probability $(1 - \delta)$ we have
	\begin{eqnarray*}
		\abs{v_i^{(\tau)} - v_i} &=& \abs{\frac{1}{x^{(\tau)}_i} - \frac{1}{x_i}} = \abs{\frac{x_i - x^{(\tau)}_i}{x^{(\tau)}_i x_i}}\\
		&\le& \frac{\eps_\tau/8}{x^{(\tau)}_i x_i} \quad (\text{holds with prob.\ $(1 - \delta)$ by (\ref{eq:b-2})})\\
		&\le& \frac{\eps_\tau/8}{1/2 \cdot 3/8} \\
		&<& \eps_\tau,
	\end{eqnarray*}
	where the second inequality holds since (i) $x_i = \frac{1}{1+v_i} \ge 1/2$ given $v_i \in [0, 1]$, and (ii) $x^{(\tau)}_i \ge 3/8$ given $\abs{x^{(\tau)}_i - x_i} < {\epsilon_{\tau}}/{8} < 1/8$.
\end{proof}

\section{Proof of Theorem~\ref{thm:improve}}
\label{app:proof-thm-improve}

First, we have the following two observations for the procedure $\ExploreS$. 
\begin{observation}[\cite{AAGZ19}]
\label{ob:geo}
	For any $i \in S$, $f_i \sim \Geo(1 / (1 + v_i))$.
\end{observation}

\begin{observation}
\label{ob:explore-time}
	The number of pulls made in $\ExploreS(S)$ is $(X + 1)$ where \(X \sim \Geo(1 / (1 + \sum_{i \in S} v_i))\).
\end{observation}

\paragraph{Correctness.} We define the following event which we will condition on in the rest of the proof. 

\begin{equation}
\E_2 \triangleq \{\forall{\tau \ge 0, i \in I_{\tau}} : \abs{v^{(\tau)}_i - v_i} < \epsilon_{\tau}\}
\end{equation}

We have the following lemma regarding $\E_2$.

\begin{lemma}
	\label{lem:E-2}
	$\Pr[\E_2] \ge 1 - {\delta}/{2}$.
\end{lemma}

\begin{proof}
	By Observation~\ref{ob:geo} and Lemma~\ref{lem:geo-additive}, we have that for any $\tau \ge 0$ and $i \in I_{\tau}$, it holds that
	\begin{equation}
	\Pr\left[\abs{v^{(\tau)}_i - v_i} \ge \epsilon_{\tau}\right] \le 2 \exp\left(-\frac{\epsilon^2_{\tau} T_{\tau}}{8}\right) \le \frac{\delta}{8 N (\tau + 1)^2} \,.
	\end{equation}
	By a union bound we have
	\begin{eqnarray*}
	\Pr[\bar{\E_2}]  &\le& \sum_{\tau = 0}^{\infty}\Pr\left[\abs{v^{(\tau)}_i - v_i} \ge \epsilon_{\tau}\right] \nonumber \\
	&\le& \sum_{\tau = 0}^{\infty} \sum_{i \in I_{\tau}} \frac{\delta}{8 N (\tau + 1)^2} \le \frac{\delta}{2}\,.
	\end{eqnarray*}
\end{proof}

By the same arguments as that for Theorem~\ref{thm:basic}, we can show that Algorithm~\ref{alg:improve} returns the correct answer given that event $\E_2$ holds. Then by Lemma~\ref{lem:E-2}, Algorithm~\ref{alg:improve} succeeds with probability at least $1 - {\delta}/{2}$.

\paragraph{Pull Complexity.}  Now we turn to the number of pulls that Algorithm~\ref{alg:improve} makes. For any $i \in I$ we again define
\begin{equation}
\tau(i) \triangleq \min \left\{\tau \ge 0 : \epsilon_{\tau} \le \frac{\Delta_i}{32 K} \right\}.
\end{equation}

The following lemma is identical to Lemma~\ref{lem:stop} in the proof for Theorem~\ref{thm:basic}.

\begin{lemma}
	\label{lem:stop-2}
	In Algorithm~\ref{alg:improve}, for any item $i \in I$, we have $i \not\in I_\tau$ for any $\tau > \tau(i)$.
\end{lemma}

We next show that Algorithm~\ref{alg:improve} will not make too many pulls in each round. 

The following lemma is a direct consequence of Observation~\ref{ob:explore-time} and Lemma~\ref{lem:geo-multi} (setting $\lambda = 5$).
\begin{lemma}
	\label{lem:round-concentration}
	For any $T > 0$, let random variables $X_t\ (t = 1, \ldots, T)$ be the number of pulls made at the $t$-th call $\ExploreS(S)$. We have
	\begin{equation*}
	\Pr\left[\sum_{t = 1}^{T} X_t \ge 5 \left(1 + \sum_{i \in S} v_i \right) T \right] \le  \exp(-2T)\,.
	\end{equation*}
\end{lemma}

For each round $\tau$, applying Lemma~\ref{lem:round-concentration} with $T = T_\tau - T_{\tau-1}$ for each $S \in \{S^{\tau}_1, \ldots, S^\tau_{m_\tau}\}$ we get
\begin{eqnarray*}
	&&\Pr\left[\sum_{t = 1}^{T_{\tau} - T_{\tau - 1}} X_t   \ge  5 \left(1 + \sum_{i \in S} v_i\right)\left(T_{\tau} - T_{\tau - 1}\right)\right] \\
	& \le & \exp(-2(T_{\tau} - T_{\tau - 1})) \\
	& \le & \exp\left(-T_{\tau}\right) \le  \frac{\delta}{8 N (\tau + 1)^2},
\end{eqnarray*}
where in the second inequality we have used the fact $T_\tau - T_{\tau - 1} \ge T_\tau/2$ (by the definition of $T_\tau$). 

By a union bound over $S \in \{S^{\tau}_1, \ldots, S^{\tau}_{m_\tau}\}$ and $\tau \ge 0$, with probability 
\begin{equation}
\label{eq:j-0}
1 - \sum_{\tau \ge 0} \left( m_\tau \cdot \frac{\delta}{8 N (\tau+1)^2} \right) \ge 1 - \frac{\delta}{2},
\end{equation}
the total number of pulls made by Algorithm~\ref{alg:improve} is bounded by
\begin{eqnarray}
&& 5 \sum_{\tau \ge 0, I_{\tau} \neq \emptyset} \left(\left\lceil\frac{\abs{I_{\tau}}}{K}\right\rceil + \sum_{i \in I_{\tau}} v_i\right) \left(T_{\tau} - T_{\tau - 1}\right) \\
&\le& 5  \sum_{\tau \ge 0, I_{\tau} \neq \emptyset} \left(\frac{\abs{I_{\tau}}}{K}+ 1 + \sum_{i \in I_{\tau}} v_i\right) \left(T_{\tau} - T_{\tau - 1}\right) \nonumber \\
&=& 5 \sum_{\tau \ge 0, I_{\tau} \neq \emptyset} \left(T_{\tau} - T_{\tau - 1}\right)\\
&& +
5 \sum_{\tau \ge 0, I_{\tau} \neq \emptyset} \left(\sum_{i \in I_{\tau}} \left(v_i + \frac{1}{K}\right)\right) \left(T_{\tau} - T_{\tau - 1}\right). \label{eq:j-1}
\end{eqnarray}
By Lemma~\ref{lem:stop-2} we know that for any $ \tau > \bar{\tau} \triangleq \max_{i \in I}\{ \tau(i)\}$, it holds that $I_{\tau} = \emptyset$. We thus have
\begin{equation}
\label{eq:j-2}
\sum_{\tau \ge 0, I_{\tau} \neq \emptyset}(T_{\tau} - T_{\tau-1}) \le  T_{\bar{\tau}}.
\end{equation}
Again by Lemma\ref{lem:stop-2} we have
\begin{equation}
\label{eq:j-3}
\sum_{\tau \ge 0}\left(\sum_{i \in I_{\tau}}\left(v_i + \frac{1}{K}\right)\right)(T_{\tau} - T_{\tau - 1})  \le \sum_{i \in I}\left(v_i + \frac{1}{K}\right)T_{\tau(i)}.
\end{equation}

Combining (\ref{eq:j-0}), (\ref{eq:j-1}), (\ref{eq:j-2}), (\ref{eq:j-3}) and Lemma~\ref{lem:E-2}, we have that with probability $1 - (\delta/2 + \delta/2) = 1-\delta$, the total number of pulls made by Algorithm~\ref{alg:improve} is bounded by 
\begin{equation}
\label{eq:j-4}
O \left(T_{\bar{\tau}} + \sum_{i \in I_{\tau}}\left(v_i + \frac{1}{K}\right)T_{\tau(i)}\right)\,.
\end{equation}
By the definitions of $\tau(i)$ and $T_\tau$ we have $$T_{\tau(i)} = O\left(\frac{K^2}{\Delta^2_i} \cdot \ln\left(\frac{N}{\delta} \tau(i) \right)\right),$$ where $\tau(i) = O(\ln(K/\Delta_i)) = O(\ln (KH_2))$. Plugging these values to (\ref{eq:j-4}) we can bound the total number of pulls by $= O\left(K^2 H_2 \ln\left(\frac{N}{\delta}  \ln (KH_2)\right)\right)$.

\paragraph{Running Time.}  The analysis of the running time of Algorithm~\ref{alg:improve} is very similar as that for Algorithm~\ref{alg:basic}. The main difference is that the time complexity for each call of \ExploreS\ is bounded $O(N \beta)$ (instead of $O(\beta)$ for $\Explore$) in the worst case, where $\beta$ is the number of pulls in the call.  This is why the first term in the time complexity in Theorem~\ref{thm:improve} is $NT$ instead of $T$ as that in Theorem~\ref{thm:basic}.  The second term concerning the $\Prune$ subroutine is the same as that in Theorem~\ref{thm:basic}.

\section{Proof of Theorem~\ref{thm:lb} (The Lower Bound)}
\label{app:lb-proof}

We consider the following two input instances.  Let $\delta  \in \left(0, \frac{1}{4 K}\right)$ be a parameter.

\begin{itemize}
\item {\em Instance $I_1$}.  $I_1$ contains $N = K$ items with rewards $r_1 = \ldots = r_{K-1} = 1, r_K = \frac{1 - \delta}{2 - \delta}$, and preferences $v_1 = \ldots = v_{K-1} = \frac{1}{K-1}, v_K = 1$.

\item {\em Instance $I_2$}. $I_2$ contains $N = K$ items with rewards $r_1 = \ldots = r_{K-1} = 1, r_K = \frac{1 - \delta}{2 - \delta}$, and preferences $v_1 = \frac{1}{K-1} - 2\delta, v_2 = \ldots = v_{K-1} = \frac{1}{K-1}, v_K = 1$.
\end{itemize}

Before proving Theorem~\ref{thm:lb}, we first bound the instance complexities of $I_1$ and $I_2$.

\paragraph{Instance complexity of $I_1$.}
The optimal expected reward of $I_1$ is $1/2$, achieved on the set $[K - 1]$. Indeed, all items from $[K - 1]$ should be included in the best assortment since their rewards are all $1$, and this already gives an expected reward of 
$$\sum_{i \in [K-1]} \frac{1 \cdot v_i}{1 + \sum_{j \in [K-1]} v_j} = \frac{1}{2}.$$ 
While the reward of Item $K$ is $\frac{1 - \delta}{2 - \delta} < \frac{1}{2}$, and thus Item $K$ should be excluded in the best assortment. 

By Definition~\ref{def:gap}, we have
\begin{equation*}
%\label{eq:m-1}
		\Delta_K = \frac{1}{2} - \frac{1 - \delta}{2 - \delta} = \frac{2 - \delta - 2 + 2 \delta}{2(2 - \delta)} \ge \frac{\delta}{4}\,.
\end{equation*}
For every $i \in [K-1]$, we have
\begin{equation*}
		\Delta_i = \min\left\{1 - \frac{1}{2}, \Delta_K\right\} = \Delta_K.
\end{equation*} 
We can thus bound 
\begin{equation}
\label{eq:m-2}
	H_2(I_1) = \sum_{i \in [K]} \frac{v_i + 1 / K}{\Delta_i^2} + \max_{i  \in [K-1]}\left\{\frac{1}{\Delta_i^2}\right\} \le \frac{4}{\Delta_K^2} \le \frac{64}{\delta^2}\,.
\end{equation}

\paragraph{Instance complexity of $I_2$.}
The optimal expected reward of $I_2$ is at least that of the assortment $[K]$, which can be bounded as
\begin{equation*}
%\label{eq:m-3}
\left(\sum_{i \in [K-1]} \frac{1 \cdot v_i}{1 + \sum_{j \in [K]} v_j}\right) + \frac{\frac{1 - \delta}{2 - \delta} \cdot v_K}{1 +  \sum_{j \in [K]} v_j} \ge \frac{1 - 2\delta}{2 - 2\delta}.  
\end{equation*}
Thus, for every $i \in [K]$, we have
\begin{equation*}
		\Delta_K \ge \frac{1 - \delta}{2 - \delta} - \frac{1 - 2\delta}{2 - 2\delta} \ge \frac{\delta}{4}\,.
\end{equation*}
We can again bound 
\begin{equation}
	H_2(I_2) = \sum_{i \in [K]} \frac{v_i + 1 / K}{\Delta_i^2} + \max_{i  \in [K]}\left\{\frac{1}{\Delta_i^2}\right\} \le \frac{4}{\Delta_K^2} \le \frac{64}{\delta^2}\,. \label{eq:m-4}
\end{equation}

By (\ref{eq:m-2}) and (\ref{eq:m-4}), to prove Theorem~\ref{thm:lb} it suffices to show the following.

\begin{lemma}
\label{lem:lb-inside}
	Any algorithm that uses less than $\frac{c}{4  \delta^2 K^2}$ pulls for $c < 10^{-4}$  outputs the wrong answer on at least one instance among $I_1$ and $I_2$ with the probability at least $0.4$.
\end{lemma}

In the rest of this section we prove Lemma~\ref{lem:lb-inside}.  
We can focus on deterministic algorithms, since for any randomized algorithm we can always fix its randomness and obtain the deterministic algorithm with the smallest error on the input. 

Let $\mathcal{T}_t = (U_1, o_1), \dotsc, (U_t, o_t)$ be the transcript of algorithm up to the $t$-th pull.  We use $g_1(\mathcal{T}_t)$ and $g_2(\mathcal{T}_t)$ to denote the probabilities of observing the transcript $\mathcal{T}_t$ on instances $I_1$ and $I_2$ respectively.  The following lemma is the key for proving Lemma~\ref{lem:lb-inside}.

\begin{lemma}
\label{lem:divergence-2}
	Let $c > 0$ and $T = \frac{c}{4 \delta^2 K^2}$. For all $\eps > 0$, we have
	\begin{equation*}
		\Pr_{\mathcal{T}_T \sim g_1}\left[ \ln \frac{g_2(\mathcal{T}_T)}{g_1(\mathcal{T}_T)} \le -(\eps + c)\right]  \le \exp\left(\frac{-\eps^2}{9 c}\right).
	\end{equation*}
\end{lemma}

To see Lemma~\ref{lem:divergence-2} implies Lemma~\ref{lem:lb-inside}, we set $\eps = \frac{1}{5}$, $c = \frac{1}{2250}$, and define event $\Q$ as
\begin{equation}
\label{eq:n-1}
	\Q \triangleq \left\{\ln \frac{g_2(\mathcal{T}_T)}{g_1(\mathcal{T}_T)} > -(\eps + c)\right\}.
\end{equation}
By Lemma~\ref{lem:divergence-2}, it holds that $\Pr_{\mathcal{T}_T \sim g_1}[\bar{Q}] \le e^{-10}$.
Let $\B$ be the event that algorithm $\A$ outputs the set $[K - 1]$.  We have
\begin{eqnarray}
\Pr_{\mathcal{T}_T \sim g_1}[\B] & = & \Pr_{\mathcal{T}_T \sim g_1}\left[\B \land \bar{\Q}~\right] + \Pr_{\mathcal{T}_T \sim g_1}[\B \land \Q] \nonumber \\
& \le &  \Pr_{\mathcal{T}_T \sim g_1}\left[\bar{\Q}~\right] + \Pr_{\mathcal{T}_T \sim g_1}[\B \land \Q] \nonumber \\
& \le & e^{-10} + \Pr_{\mathcal{T}_T \sim g_1}[\B \land \Q] \nonumber \\
& = & e^{-10} + \sum_{\mathcal{T}_T : \B \land \Q} g_1(\mathcal{T}_T) \nonumber\\
& \stackrel{\eqref{eq:n-1}}{\le} & e^{-10} + e^{\eps + c}\sum_{\mathcal{T}_T : \B \land \Q} g_2(\mathcal{T}_T) \nonumber\\
& \le & e^{-10} + e^{\eps + c} \Pr_{\mathcal{T}_T \sim g_2}[\B] \nonumber \\
& = & e^{-10} + e^{\eps + c} - e^{\eps + c} \Pr_{\mathcal{T}_T \sim g_2}\left[~\bar{\B}~\right]. \nonumber
\end{eqnarray}
Therefore, we have
\begin{equation}
	\Pr_{\mathcal{T}_T \sim g_1}[\B] + e^{\eps + c}\Pr_{\mathcal{T}_T \sim g_2}\left[~\bar{\B}~\right] \le e^{-10} + e^{\eps + c},  \nonumber
\end{equation}
and consequently, 
\begin{equation}
	\min\left\{\Pr_{\mathcal{T}_T \sim g_1}[\B], \Pr_{\mathcal{T}_T \sim g_2}\left[~\bar{\B}~\right]\right\} \le \frac{e^{-10} + e^{\eps + c}}{1 + e^{\eps + c}} \le 0.6.
	\label{eq:o-1}
\end{equation}

\eqref{eq:o-1} indicates that one of the followings hold: (1) Event $\B$ holds with probability at most $0.6$ when $\mathcal{T}_T \sim g_1$, and (2) Event $\bar{\B}$ holds with probability at most $0.6$ when  $\mathcal{T}_T \sim g_2$.  In the first case, it indicates that algorithm $\A$ errors on input instance $I_1$ with probability at least $0.4$. In the second case, it indicates that  algorithm $\A$ errors on input instance $I_2$ with probability at least $0.4$. 
  
We now prove Lemma~\ref{lem:divergence-2}.

\begin{proof} (of Lemma~\ref{lem:divergence-2}) 
We define a sequence of random variables $Z_0, Z_1, \ldots, Z_T$ when the transcript $\mathcal{T}_t\ (0 \le t \le T)$ is produced by applying algorithm $\A$ on the input instance $I_1$:
\begin{equation*}
\label{eq:p-1}
	Z_t = \ln \frac{g_2(\mathcal{T}_t)}{g_1(\mathcal{T}_t)}.
\end{equation*} 
Let $V_t = \sum_{i \in U_t} v_i$.  $Z_i$ has the following properties.
\begin{itemize}
	\item  If $1 \not\in U_t $, then $Z_t - Z_{t - 1} = 0$, and  $\bE[Z_t - Z_{t - 1} \mid  Z_{t - 1}] = 0$.
	
	\item If $1 \in U_t$, then with probability $\frac{1 + V_t - v_1}{1 + V_t}$,
	\begin{equation*}
	Z_t - Z_{t - 1} = -\ln\left(1 - \frac{2\delta}{1 + V_t}\right),
	\end{equation*}
	and with probability $\frac{v_1}{1 + V_t}$, 
	\begin{equation*}
	Z_t - Z_{t - 1} =	-\ln\left(1 - \frac{2\delta}{1 + V_t}\right) + \ln\left(1  - \frac{2\delta}{v_1}\right)\,.
	\end{equation*}
We thus have 
\begin{eqnarray}
		\bE[Z_t - Z_{t - 1} \mid Z_{t - 1}] &=& -\ln\left(1 - \frac{2\delta}{1 + V_t}\right) \nonumber \\
		&&+ \frac{v_1}{1 + V_t} \ln\left(1 - \frac{2\delta}{v_1}\right).
		\label{eq:q-1}
\end{eqnarray}
\end{itemize}
Using inequalities $\ln(1 + x) \le x$ and $\ln(1 - x) \ge -x - x^2$ for $x \in [0, 0.5]$, and noting that $2\delta / v_1 = 2 \delta (K - 1) \le 0.5$, we have
\begin{eqnarray}
	\eqref{eq:q-1} &\ge& \frac{2\delta}{1 + V_t} - \frac{2\delta}{1 + V_t} - \frac{4\delta^2}{(1 + V_t)v_1} \ge -\frac{4 \delta^2 (K - 1)}{1 + V_t} \nonumber \\ 
	&\ge& -4 \delta^2 K^2\,.
	\label{eq:q-2}
\end{eqnarray}
Note that in the case that $1 \not\in U_t$, the inequality $\bE[Z_t - Z_{t - 1} \mid Z_{t - 1}] = 0 \ge -{4\delta^2 K^2}$ holds trivially.

We can also bound the difference of two adjacent variables in the sequence $\{Z_0, Z_1, \ldots, Z_T\}$.
\begin{equation}
\label{eq:q-3}
	\abs{Z_t - Z_{t - 1}} \le \abs{\ln\left(1 - \frac{2\delta}{1 + V_t}\right)} + \abs{\ln\left(1 - \frac{2\delta}{v_1}\right)} \le {2\delta K}\,.
\end{equation}

Define $Z^{\prime}_t \triangleq Z_t + 4 \delta^2 K^2 t$.  By \eqref{eq:q-2} it follows that $Z'_t$ is a submartingale and satisfies 
\begin{equation}
	\bE[Z^{\prime}_{t + 1} \mid Z^{\prime}_t] \ge Z^{\prime}_t\,.
\end{equation}
By \eqref{eq:q-3} and the fact that $\delta <  \frac{1}{4 K}$, we have
\begin{equation}\label{eq:q-4}
	\abs{Z^{\prime}_t - Z^{\prime}_{t - 1}} \le 4\delta^2 K^2 + 2\delta K\le 3\delta K\,.
\end{equation}
By \eqref{eq:q-4} and Azuma's inequality (Lemma~\ref{lem:azuma}), for $T = \frac{c}{4 \delta^2 K^2}$, we get
\begin{eqnarray*}
&&	\Pr_{\mathcal{T}_T \sim g_1}[Z_T \le -(\eps + c)] = \Pr_{\mathcal{T}_T \sim g_1}[Z^{\prime}_T \le -\eps] \\
	&<& \exp\left(\frac{-\eps^2}{18 T \delta^2 K^2}\right) \le \exp\left(\frac{-2\eps^2}{9 c}\right)\,.
\end{eqnarray*}
\end{proof}

The lemma follows from (\ref{eq:i-1}) and (\ref{eq:i-2}).

\end{document}